\documentclass[sigconf]{aamas}  
\AtBeginDocument{%
  \providecommand\BibTeX{{%
    \normalfont B\kern-0.5em{\scshape i\kern-0.25em b}\kern-0.8em\TeX}}}
\usepackage{booktabs}

\usepackage{graphicx}

\usepackage{amsfonts}
\usepackage{isomath}
\usepackage{amsthm}
\usepackage{booktabs}
\usepackage{hyperref}

\usepackage{import}
\usepackage{algorithm}
\usepackage[noend]{algorithmic}
\usepackage{tabularx}
\usepackage{flushend}
\setcopyright{ifaamas}  
\acmDOI{doi}  
\acmISBN{}  
\acmConference[AAMAS'20]{Proc.\@ of the 19th International Conference on Autonomous Agents and Multiagent Systems (AAMAS 2020), B.~An, N.~Yorke-Smith, A.~El~Fallah~Seghrouchni, G.~Sukthankar (eds.)}{May 2020}{Auckland, New Zealand}  
\acmYear{2020}  
\copyrightyear{2020}  
\acmPrice{}  
\usepackage{amsmath}
\DeclareMathOperator*{\argminA}{arg\,max}

\begin{document}

\title{Leveraging Communication Topologies Between \\ Learning Agents in Deep Reinforcement Learning}  

\author{Dhaval Adjodah$^1$, Dan Calacci$^1$, Abhimanyu Dubey$^1$, \\ Anirudh Goyal$^2$, P. M. Krafft$^4$, Esteban Moro$^{1,3}$, Alex Pentland$^1$}
\affiliation{%
  \institution{$^1$ Massachusetts Institute of Technology \\ 
  $^2$ MILA/Universit\'e de Montr\'eal \\ 
  $^3$ Universidad Carlos III de Madrid \\
  $^4$ Oxford Internet Institute, University of Oxford
  }
}

\begin{abstract}  
A common technique to improve learning performance in deep reinforcement learning (DRL) and many other machine learning algorithms is to run multiple learning agents in parallel. A neglected component in the development of these algorithms has been how best to arrange the learning agents involved to improve distributed search. Here we draw upon results from the networked optimization literatures suggesting that arranging learning agents in communication networks other than fully connected topologies (the implicit way agents are commonly arranged in) can improve learning. We explore the relative performance of four popular families of graphs and observe that one such family (Erdos-Renyi random graphs) empirically outperforms the de facto fully-connected communication topology across several DRL benchmark tasks. Additionally, we observe that 1000 learning agents arranged in an Erdos-Renyi graph can perform as well as 3000 agents arranged in the standard fully-connected topology, showing the large learning improvement possible when carefully designing the topology over which agents communicate. We complement these empirical results with a theoretical investigation of why our alternate topologies perform better. Overall, our work suggests that distributed machine learning algorithms could be made more effective if the communication topology between learning agents was optimized\footnote{Correspondence to Dhaval Adjodah (dhaval@mit.edu). Code available at  \href{https://github.com/d-val/NetES}{github.com/d-val/NetES}}.
\end{abstract}


\maketitle


\section{Introduction}

    Implementations of deep reinforcement learning (DRL) algorithms have become increasingly distributed, running large numbers of parallel sampling and training nodes. For example, AlphaStar runs thousands of parallel instances of Stracraft II on TPU's \cite{alphastarblog}, and OpenAI Five runs on 128,000 CPU cores at the same time \cite{OpenAI_dota}. 

    Such distributed algorithms rely on an implicit communication network between the processing units being used in the algorithm.  These units pass information such as data, parameters, or rewards between each other, often through a central controller. For example, in the popular A3C \cite{Mnih2016} reinforcement learning algorithm, multiple `workers' are spawned with local copies of a global neural network, and they are used to collectively update the global network. These workers can either be viewed as implementing the parallelized form of an algorithm, or they can be seen as a type of multi-agent distributed optimization approach to searching the reward landscape for parameters that maximize performance.

    In this work, we take the latter approach of thinking of the `workers' as separate agents that search a reward landscape more or less efficiently. We adopt such an approach because it allows us to consider improvements studied in the field of multi-agent optimization \cite{ferber1999multi}, specifically the literatures of networked optimization (optimization over networks of agents with local rewards) \cite{nedic2017network,nedic201010,nedic2011asynchronous} and collective intelligence (the study of mechanisms of how agents learn, influence and collaborate with each other) \cite{wolpert1999introduction,woolley2010evidence}. 
    
    These two literatures suggest a number of different ways to improve such multi-agent optimization, and, in this work, we choose to focus on one of main ways to do so: optimizing the topology of communication between agents (i.e. the local and global characterization of the connections between agents used to communicate data, parameters, or rewards with).

    We focus on communication topology because it has been shown to result in increased exploration, higher reward, and higher diversity of solutions in both simulated high-dimensional optimization problems \cite{Lazer2007} and human experiments \cite{Barkoczi2016}, and because, to the best of our knowledge, almost no prior work has investigated how the topology of communication between agents affects learning performance in distributed DRL.  
    
    Here, we empirically investigate whether using alternate communication topologies between agents could lead to improving learning performance in the context of DRL. The two topologies that are almost always used in DRL are either a complete (fully-connected) network, in which all processors communicate with each other; or a star network---in which all processors communicate with a single hub server, which is, in effect, a more efficient, centralized implementation of the complete network (e.g., \cite{scott2016bayes}). Our hypothesis is that using other topologies than fully-connected will lead to learning improvements. 
    
    Given that network effects are sometimes only significant with large numbers of agents, we choose to build upon one of the DRL algorithms most oriented towards parallelizability and scalability: Evolution Strategies~\cite{rechenberg1973evolution,schwefel1977numerische,wierstra2014natural}, which has recently been shown to scale-up to tens of thousands of agents \cite{Salimans2017}. 
    
    We introduce Networked Evolution Strategies (NetES), a networked decentralized variant of ES. NetES, like many DRL algorithms and evolutionary methods, relies on aggregating the rewards from a population of processors that search in parameter space to optimize a single global parameter set. Using NetES, we explore how the communication topology of a population of processors affects learning performance.

    Key aspects of our approach, findings, and contributions are as follows:
\begin{itemize}
    	\item We introduce the notion of communication network topologies to the ES paradigm for DRL tasks.
    	\item We perform an ablation study using various baseline controls to make sure that any improvements we see come from using alternate topologies and not other factors.
    	\item We compare the learning performance of the main topological families of communication graphs, and observe that one family (Erdos-Renyi graphs) does best.
    	\item Using an optimized Erdos-Renyi graph, we evaluate NetES on five difficult DRL benchmarks and find large improvements compared to using a fully-connected communication topology. We observe that our 1000-agent Erdos-Renyi graph can compete with 3000 fully-connected agents.
    	\item We derive an upper bound which provides theoretical insights into why alternate topologies might outperform a fully-connected communication topology. We find that our upper bound only depends on the topology of learning agents, and not on the reward function of the reinforcement learning task at hand, which indicates that our results likely will generalize to other learning tasks.
\end{itemize}

        \begin{table}
            \centering
        \resizebox{1\linewidth}{!}{%
        \begin{tabular}{@{}lllll@{}}
        \toprule
        \textbf{Type} & \textbf{Task}  & \textbf{Fully-connected} & \textbf{Erdos} & \textbf{Improv. \%} \\ \midrule
        MuJoCo        & Ant-v1         & 4496                     & 4938           & \textbf{9.8}            \\
        MuJoCo        & HalfCheetah-v1 & 1571                     & 7014           & \textbf{346.3}          \\
        MuJoCo        & Hopper-v1      & 1506                     & 3811           & \textbf{153.1}          \\
        MuJoCo        & Humanoid-v1    & 762                      & 6847           & \textbf{798.6}          \\
        Roboschool    & Humanoid-v1     & 364                      & 429            & \textbf{17.9}           \\ \bottomrule
        \end{tabular}%
        }
        \caption{Improvements from Erdos-Renyi networks with 1000 nodes compared to fully-connected networks. }
        \label{tab:reward_asym}
        \end{table}

\begin{figure}[t]
                  \centering
                  \includegraphics[width=\linewidth]{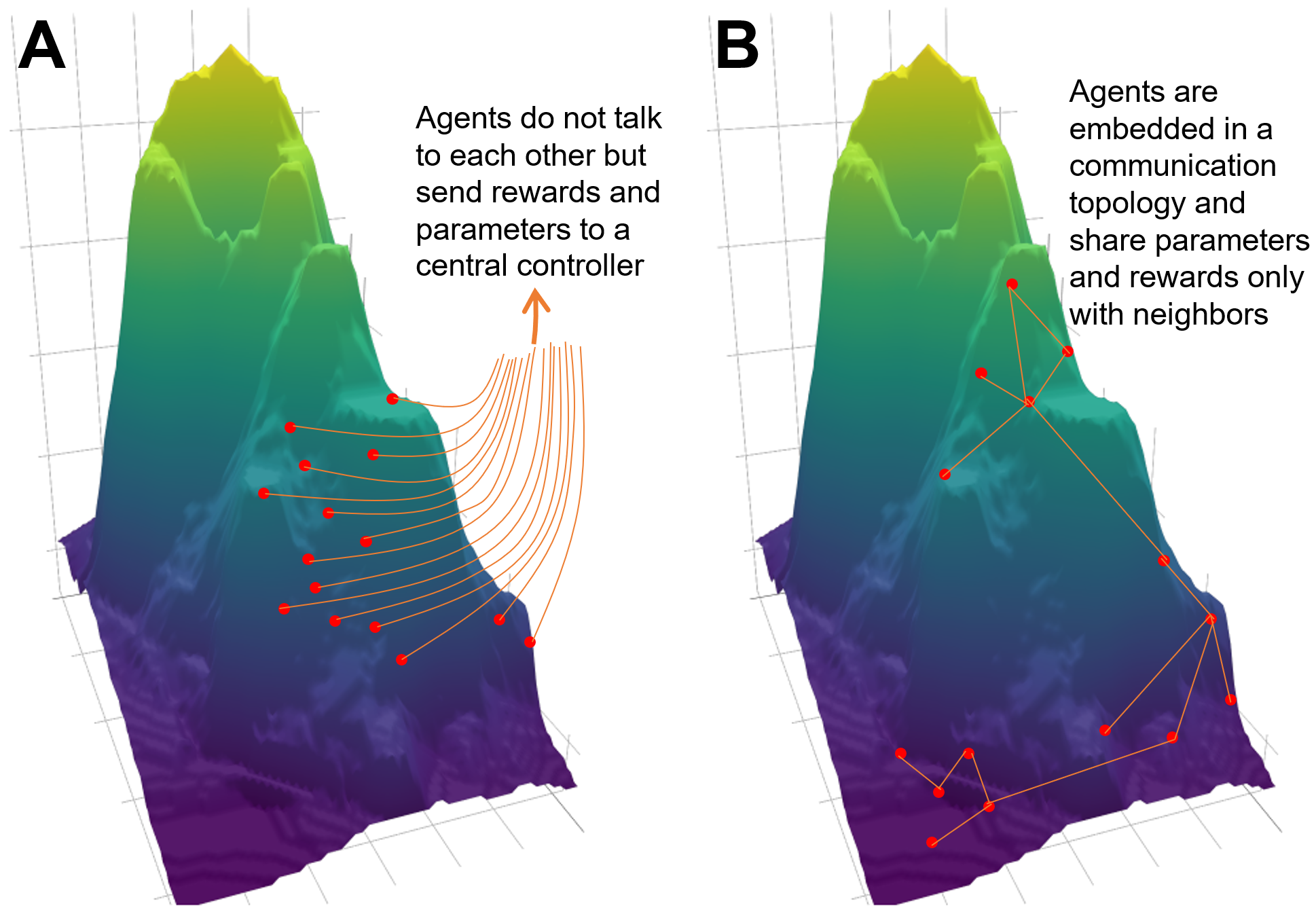}
                  \caption{Learning in DRL can be visualized with agents (red dots) searching a reward landscape for the parameter set (location) that leads to the highest reward.      
                  A: In most DRL algorithms, including ES, agents are searching the same local area. Because the controller receives information from all agents, and then broadcasts a new parameter to all other agents, agents are, in effect communicating in a fully connected network.
                  B: In NetES, the same number of agents are embedded in a communication topology over which they share data. This leads to a more distributed search where each cluster of agents focuses on a different part of the landscape. 
                  }
    \label{fig:illustration}
    \end{figure}

 \section{Preliminaries}
     \subsection{Evolution Strategies for Deep RL}
        As discussed earlier, given that network effects are generally only significant with large numbers of agents, we choose to build upon one of the DRL algorithms most oriented towards parallelizability and scalability: Evolution Strategies.

        We begin with a brief overview of the application of the Evolution Strategies (ES)~\cite{schwefel1977numerische} approach to DRL, following Salimans et al.~\cite{Salimans2017}. Evolution Strategies is a class of techniques to solve optimization problems by utilizing
         a derivative-free parameter update approach. The algorithm proceeds by selecting a fixed
         model, initialized with a set of weights $\matrixsym \theta$ (whose distribution $p_\phi$ is parameterized by parameters $\phi$), and an objective (reward) function $R(\cdot)$ defined externally by the DRL task being solved. The ES algorithm then
         maximizes the average objective value $\mathbb E_{\matrixsym \theta \sim p_\phi} R(\matrixsym \theta)$,
         which is optimized with stochastic gradient ascent. The score function estimator for
         $\nabla_\phi \mathbb E_{\matrixsym \theta \sim p_\phi} R(\matrixsym \theta)$ is similar to REINFORCE~\cite{williams1992simple}, given by $\nabla_\phi \mathbb E_{\matrixsym \theta \sim p_\phi} R(\matrixsym \theta) =
         \mathbb E_{\matrixsym \theta \sim p_\phi}[R(\matrixsym \theta) \nabla_\phi \log p_\phi (\matrixsym \theta)]$.

         The update equation used in this algorithm for the parameter $\matrixsym \theta$ at any iteration $t+1$, for an appropriately chosen learning rate $\alpha$ and noise standard deviation $\sigma$, is a discrete approximation to the gradient:
         \begin{equation}
         {\matrixsym \theta}^{(t+1)}= {\matrixsym \theta}^{(t)} + \frac{\alpha}{N \sigma^2} \sum_{i=1}^N \big(R(\matrixsym \theta^{(t)} +
         \sigma \matrixsym \epsilon_i^{(t)}) \cdot \sigma \matrixsym \epsilon_i^{(t)} \big) \label{eqn:basic_update_rule}
         \end{equation}
         This update rule is implemented by spawning a collection of $N$ agents at every iteration $t$, with perturbed versions of ${\matrixsym \theta}^{(t)}$, i.e. $\{ ({\matrixsym \theta}^{(t)} +
          \sigma {\matrixsym \epsilon}_1^{(t)}), ..., ({\matrixsym \theta}^{(t)} + \sigma {\matrixsym \epsilon}_N^{(t)}) \}$ where $\epsilon \sim \mathcal{N}(0,\,\emph{I})$.
         The algorithm then calculates ${\matrixsym \theta}^{(t+1)}$ which is broadcast again to all agents, and the process is repeated.

         In summary, a centralized controller holds a global parameter $\theta$, records the perturbed noise ${\matrixsym \epsilon}_i^{(t)}$ used by \textit{all}  agents, collects rewards from \textit{all}  agents at the end of an episode, calculates the gradient and obtains a new global parameter $\theta$.
         Because the controller receives information from all agents, and then broadcasts a new parameter to all other agents, each agent is in effect communicating (through the controller) with all other agents. 

         This means that the de facto communication topology used in Evolution Strategies (and all other DRL algorithms that use a central controller) is a fully-connected network. Our hypothesis is that using alternate communication topologies between agents will lead to improved learning performance. 
         
         So far, we have assumed that all agents start with the same global parameter $\matrixsym \theta^{(t_0)}$. When each agent $i$ starts with a different parameter ${\matrixsym \theta}_i^{(t_o)}$, Equation \ref{eqn:basic_update_rule} has to be generalized. In the case when all agents start with the same parameter, Equation \ref{eqn:basic_update_rule} can be understood as having each agent taking a weighted average of the differences (perturbations) between their last local parameter copy and the perturbed copies of each agent, (the differences being $\sigma\epsilon^{(t)}_i = ((\theta^{(t)} + \sigma\epsilon^{(t)}_i) - \theta^{(t)})$). The weight used in the weighted average is given by the reward at the location of each perturbed copy, $R(\matrixsym \theta^{(t)} + \sigma \matrixsym \epsilon_i^{(t)})$.
         
         When agents start with different parameters, the same weighted average is calculated: because each agent now has different parameters, this difference between agent $i$ and $j$'s parameters is $((\theta^{(t)}_i + \sigma\epsilon^{(t)}_i) - \theta^{(t)}_j)$. The weights are still $R(\matrixsym \theta^{(t)} + \sigma \matrixsym \epsilon_i^{(t)})$). In this notation, Equation \ref{eqn:basic_update_rule} is then:
     \begin{multline}
     {\matrixsym \theta}_j^{(t+1)}= {\matrixsym \theta}_j^{(t)} +  \frac{\alpha}{N \sigma^2} \sum_{i=1}^N \Big(R(\matrixsym \theta^{(t)}_i +
     \sigma \matrixsym \epsilon_i^{(t)}) \cdot (\matrixsym \theta^{(t)}_i + \sigma \matrixsym \epsilon_i^{(t)} -
     \matrixsym \theta^{(t)}_j)\Big) \hfill
      \label{eqn:update_rule_fully_difference_starting}
     \end{multline}
     
     It is straightforward to show that Equation \ref{eqn:update_rule_fully_difference_starting} reduces to Equation \ref{eqn:basic_update_rule} when all agents start with the same parameter. As we will show, generalizing this standard update rule further to handle alternate topologies will be straightforward. 
         
        \begin{figure*}
                 \begin{minipage}{.32\textwidth}
                   \centering
                   \includegraphics[width=\linewidth]{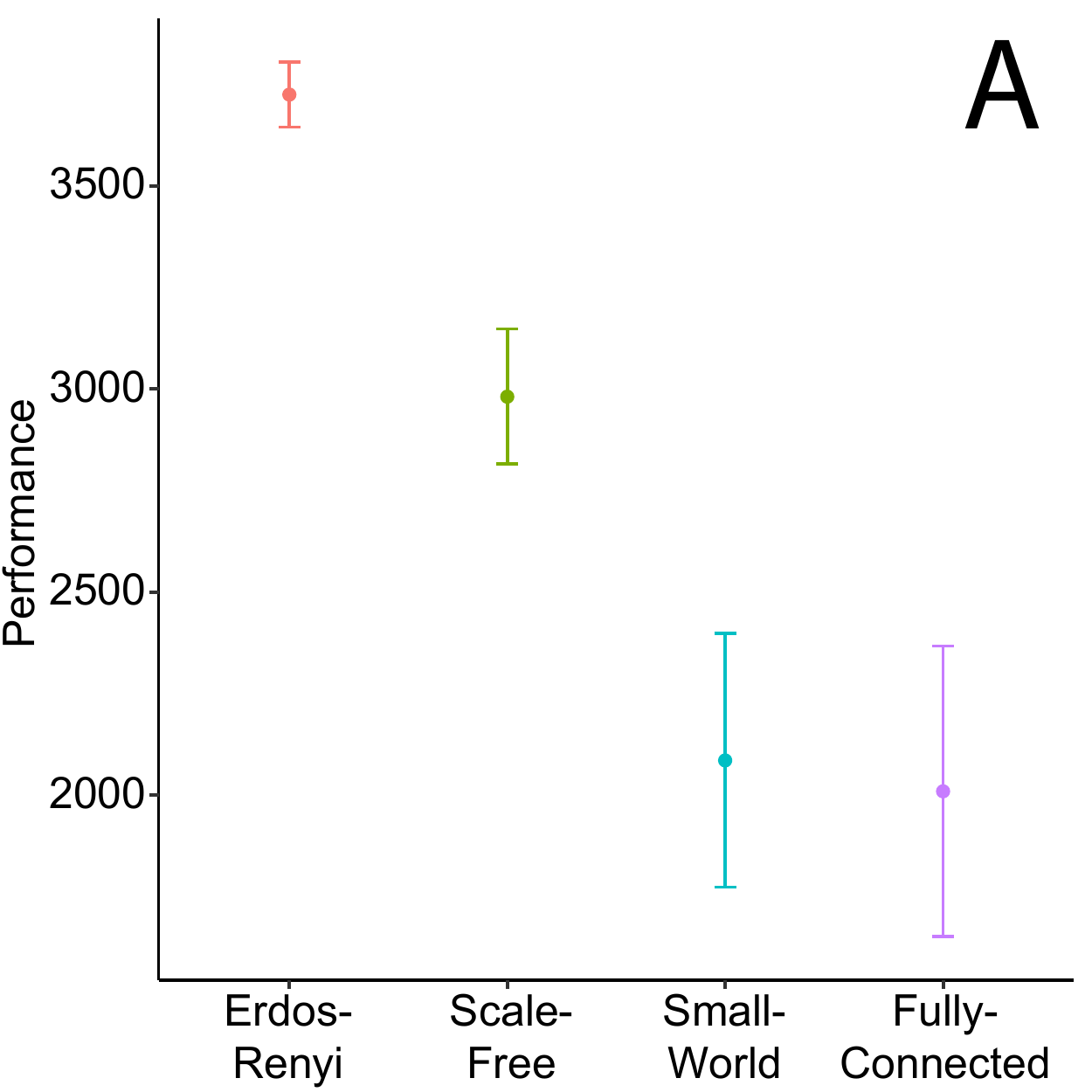}
                 \end{minipage}
                    \begin{minipage}{.32\textwidth}
                      \centering
                      \includegraphics[width=\linewidth]{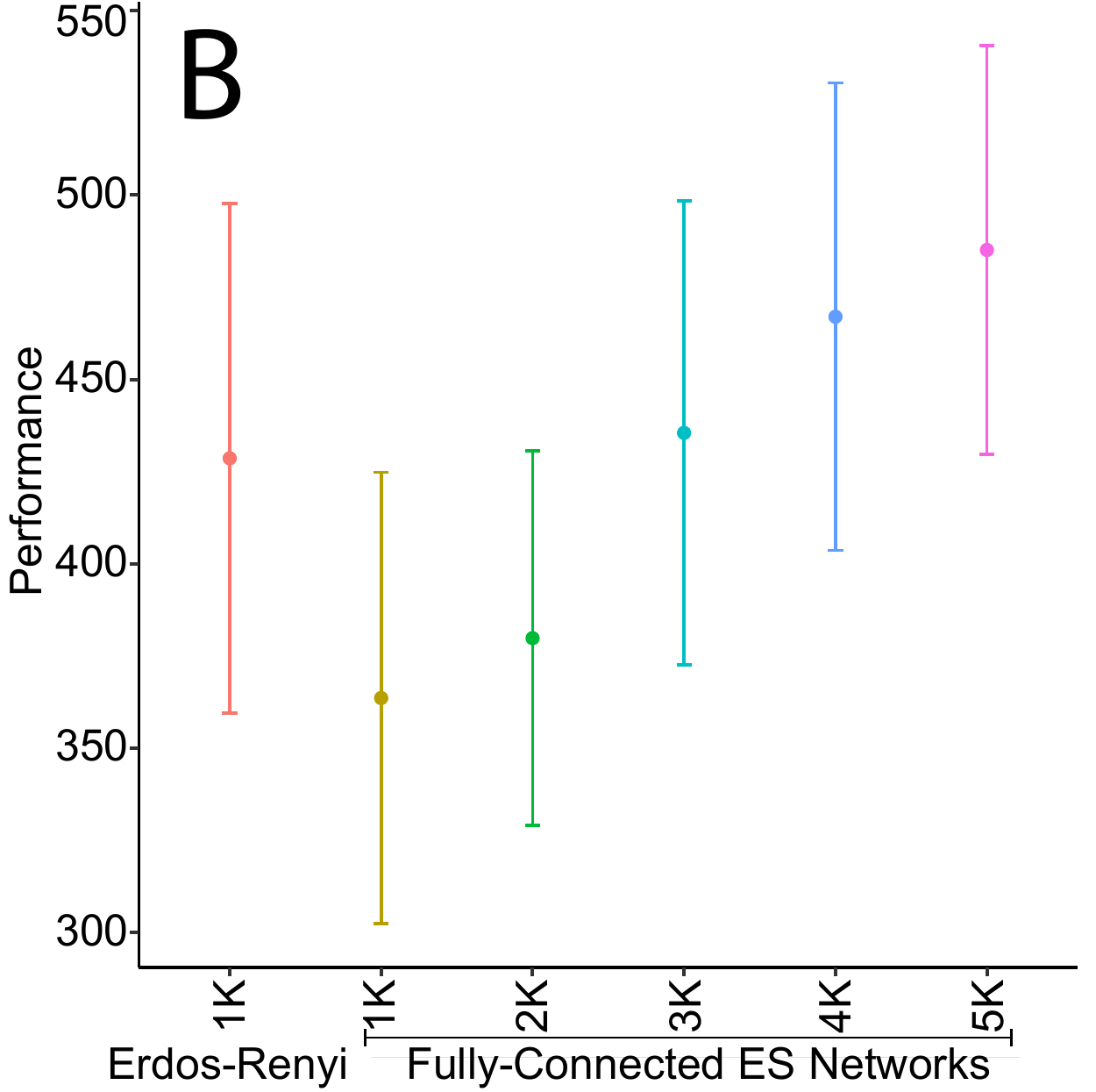}
                    \end{minipage}
                \begin{minipage}{0.32\textwidth}
                  \centering
                  \includegraphics[width=\linewidth]{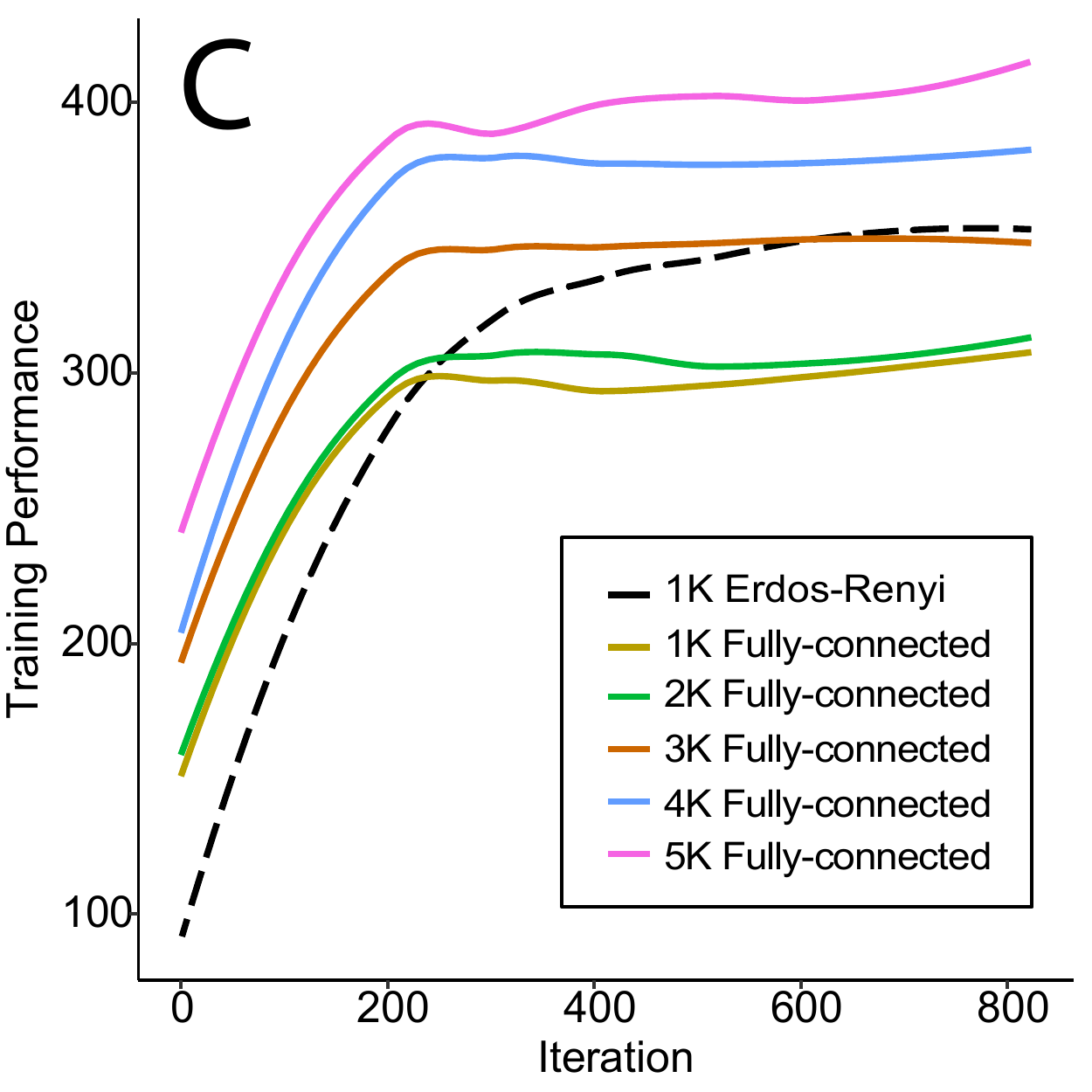}
                  \end{minipage}
                  \caption{
                  \textbf{A}: Learning performance on all network families: Erdos-Renyi graphs do best, fully-connected graphs do worst (MuJoCo Ant-v1 task with small networks of 100 nodes).
                  \textbf{B}: Evaluation results for Erdos-Renyi graph with 1000 agents compared to fully-connected networks with varying network sizes (RoboSchool Humanoid-v1).
                  \textbf{C}: Comparing Erdos-Renyi graph with 1000 agents to fully-connected networks with varying network sizes on training (not evaluation metric) performance (Roboschool Humanoid-v1).
                  \textbf{All}: Error bars represent 95\% confidence intervals.
                  }
    \label{fig:combined1}
    \end{figure*}
    
\section{Problem Statement}

The task ahead is to take the standard ES algorithm and operate it over new communication topologies, wherein each agent is only allowed to communicate with its neighbors. This would allow us to test our hypothesis that alternate topologies perform better than the de facto fully-connected topology. 

An interesting possibility for future work would be to optimize over the space of all possible topologies to find the ones that perform best for our task at hand. In this work, we take as a more tractable starting point a comparison of four popular graph families (including the fully-connected topology).

 \subsection{NetES : Networked Evolution Strategies}
We denote a network topology by ${\matrixsym A} = \{a_{ij}\}$, where $a_{ij} = 1$ if agents $i$ and $j$
     communicate with each other, and equals 0 otherwise. ${\matrixsym A}$
     represents the \textit{adjacency matrix} of connectivity, and fully characterizes the communication topology between agents.
 In a fully connected network, we have $a_{ij} = 1$ for all $i,j$.

 Using adjacency matrix ${\matrixsym A}$, it is straightforward to allow equation  \ref{eqn:update_rule_fully_difference_starting} to operate over any communication topologies:

   \begin{multline}
     {\matrixsym \theta}_j^{(t+1)}= {\matrixsym \theta}_j^{(t)} + \frac{\alpha}{N \sigma^2} \sum_{i=1}^N a_{ij}\cdot\Big(R(\matrixsym \theta^{(t)}_i +
     \sigma \matrixsym \epsilon_i^{(t)}) \cdot (\matrixsym \theta^{(t)}_i + \sigma \matrixsym \epsilon_i^{(t)} -
     \matrixsym \theta^{(t)}_j)\Big) \hfill
      \label{eqn:update_rule_4}
     \end{multline}

     Because equation \ref{eqn:update_rule_4} uses the same weighted average as in ES (equations \ref{eqn:basic_update_rule} and \ref{eqn:update_rule_fully_difference_starting}), when fully-connected networks are used (i.e. $a_{ij} = 1$) and when agents start with the same parameters, equation \ref{eqn:update_rule_4} reduces to \ref{eqn:basic_update_rule}.

    The only other change introduced by NetES is the use of periodic global broadcasts. We implemented parameter broadcast as follows: at every iteration, with a probability $p_b$, we
         choose to replace all agents' current parameters with the best agent's performing weights, and then continue training (as per Equation \ref{eqn:update_rule_4}) after that. The same broadcast techniques have been used in many other algorithms to balance local vs. global search (e.g. the `exploit' action in Population-based Training
        \cite{jaderberg2017population} replaces current agent weights with the weights that give the highest rewards).

    The full description of the NetES algorithm is shown in Algorithm \ref{alg:algo1}. 

     \begin{algorithm}
     	 \caption{Networked Evolution Strategies}
     	 \label{alg:algo1}
     \begin{algorithmic}
     	 \STATE \textbf{Input}: Learning rate $\alpha$, noise standard deviation $\sigma$,
     	 initial policy parameters ${\matrixsym \theta}_i^{(0)}$ where $i$ = 1, 2, $\ldots$, N (for N workers),
     	 adjacency matrix ${\matrixsym A}$, global broadcast probability $p_b$
     	 \STATE \textbf{Initialize}: $n$ workers with known random seeds, initial parameters ${\matrixsym \theta}_i^{(0)}$

     	 \FOR{$t$ = 0, 1, 2,$\ldots$ }
     			 \FOR {each worker $i$ = 1, 2, $\ldots$, N }
     					 \STATE Sample $\matrixsym \epsilon_j^{(t)}$ $\sim$ $\mathcal{N}(0,\,\emph{I})$

     					 \STATE Compute returns $R_i$ = $R(\matrixsym \theta^{(t)}_j +
                         \sigma \matrixsym \epsilon_j^{(t)})$
     			 \ENDFOR
     			 \STATE Sample $\beta^{(t)}$ $\sim$ $\mathcal{U}(0,\,1)$
     			 \IF { $\beta^{(t)}$ $<$ $p_b$ }
     					 \STATE Set ${\matrixsym \theta}_i^{(t+1)}$ $\gets$ $\argminA_{\matrixsym \theta_i^{(t)}} $ $R(\matrixsym \theta^{(t)}_j +
                         \sigma \matrixsym \epsilon_j^{(t)}) $
                         
     			 \ELSE
     					 \FOR {each worker $i$ = 1, 2, $\ldots$, n }
     							 \STATE Set ${\matrixsym \theta}_i^{(t+1)} \gets  {\matrixsym \theta}_i^{(t)} + \frac{\alpha}{N \sigma^2} \sum_{j=1}^N a_{ij}\cdot \newline
     							 \hfill \left(R(\matrixsym \theta^{(t)}_j +
                                 \sigma \matrixsym \epsilon_j^{(t)}) \cdot (\matrixsym \theta^{(t)}_j + \sigma \matrixsym \epsilon_j^{(t)} -
                                 \matrixsym \theta^{(t)}_i)\right)$
     					 \ENDFOR
     			 \ENDIF
     	 \ENDFOR

     \end{algorithmic}
     \end{algorithm}
     
     In summary, NetES implements three modifications to the ES paradigm: the use of alternate topologies through $a_{ij}$, the use of different starting parameters, and the use of global broadcast. In the following sections, we will run careful controls during an ablation study to investigate where the improvement in learning we observe come from. Our hypothesis is that they come mainly -- or completely -- from the use of alternate topologies. As we will show later, they do come from only the use of alternate topologies as shown in see Fig. \ref{fig:combined1}B.

    \subsection{Updating parameters over alternate topologies}
    Previous work \cite{Barkoczi2016} demonstrates that the exact form of the update rule does not matter as long as the optimization strategy is to find and aggregate the parameters with the highest reward (as opposed to, for example, finding the most common parameters many agents hold). Therefore, although our update rule is a straightforward extension of ES, we expect that our primary insight---that network topology can affect DRL---to still be useful with alternate update rules.

    Secondly, although Equation \ref{eqn:update_rule_4} is a biased gradient estimate, at least in the short term, it is unclear whether in practice we achieve a biased or an unbiased gradient estimate, marginalizing over time steps between broadcasts. This is because in the full algorithm (algorithm \ref{alg:algo1}) we implement, we combine this update rule with a periodic parameter broadcast (as is common in distributed learning algorithms - we will address this in detail in a later section), and every broadcast returns the agents to a consensus position.

    Future work can better characterize the theoretical properties of NetES and similar networked DRL algorithms using the recently developed tools of calculus on networks (e.g., \cite{acemoglu2011bayesian}). Empirically and theoretically, we present results suggesting that the use of alternate topologies can lead to large performance improvements.

    \subsection{Communication topologies under consideration}

    Given the update rule as per equation \ref{eqn:update_rule_4}, the goal is then to find which topology leads to the highest improvement.
    Because we are drawing inspiration from the study of collective intelligence and networked optimization, we use topologies that are prevalent in modeling how humans and animals learn collectively:  \begin{itemize}
      \item \textbf{Erdos-Renyi Networks:} Networks where each edge between any two nodes has a fixed independent probability of
      being present \cite{erdds1959random}, which are among the commonly used benchmark graphs for comparison in social networks \cite{mej2010networks}. 
      \item \textbf{Scale-Free Networks:} Scale-free networks, whose degree distribution follows a power law \cite{choromanski2013scale}, are commonly observed in citation and signaling biological networks \cite{barabasi1999emergence}. 
      \item \textbf{Small-World Networks:} Networks where most nodes can be reached through a small number of neighbors, resulting in the famous `six degrees of separation' \cite{travers1977experimental}. 
      \item \textbf{Fully-Connected Networks:} Networks where every node is connected to every other node.
      \end{itemize}

    We used the generative model of \cite{erdos1959random} to create Erdos-Renyi graphs, the Watts-Strogatz model \cite{watts1998collective} for Small-World graphs, and the Barabási-Albert model \cite{barabasi1999emergence} for Scale-Free networks.

    We can randomly sample instances of graphs from each family which is parameterized by the number of nodes $N$, and their degree distribution. Erdos-Renyi networks,
     for example, are parameterized by their average density $p$ ranging from 0 to 1, where 0 would
     lead to a completely disconnected graph (no nodes are connected), and 1 would
     lead back to a fully-connected graph. The lower $p$ is, the sparser a randomly generated network is. Similarly, the degree distribution of scale-free networks is defined by
     the exponent of the power distribution.
     Because each graph is generated randomly, two graphs with the same parameters will be different if they have different random seeds, even though, on average, they will have the same average degree (and therefore the same number of links).

    \subsection{Predicted improved performance of NetES}
    Through the modifications to ES we have described, we are now able to operate on any communication topology. Due to previous work in networked optimization and collective intelligence which shows that alternate network structures result in better performance, we expect NetES to perform better on DRL tasks when using alternate topologies compared to the de facto fully-connected topology. We also expect to see differences in performance between families of topologies.

\section{Related Work}
    A focus of recent DRL has been the ability to be able to run more and more agents in parallel (i.e. scalability). An early example is the Gorila framework \cite{Nair2015} that collects experiences in parallel from many agents. Another is A3C \cite{Mnih2016} that we discussed earlier. IMPALA \cite{espeholt2018impala} is a recent algorithm which solves many tasks with a single parameter set. Population Based Training \cite{jaderberg2017population} optimizes both learning weights and hyperparameters. However, in all the approaches described above, agents are organized in an implicit fully-connected centralized topology.
    
    We build on the Evolution Strategies implementation of Salimans et al. \cite{Salimans2017} which was modified for scalability in DRL. There have been many variants of Evolution Strategies over the years, such as CMA-ES~\cite{auger2005restart} which updates the covariance matrix of the Gaussian distribution, Natural Evolution strategies \cite{wierstra2014natural} where the inverse of the Fisher Information Matrix of search distributions is used in the gradient update rule, and,  Again, these algorithms implicitly use a fully-connected topology between learning agents.

    On the other hand, work in the networked optimization literature has demonstrated that the network structure of communication between nodes significantly affects the convergence rate and accuracy of multi-agent learning  \cite{nedic2017network,nedic201010,nedic2011asynchronous}. However, this work has been focused on solving global objective functions that are the sum (or average) of private, local node-based objective functions - which is not always an appropriate framework for DRL. In the collective intelligence literature, alternate network structures have been shown to result in increased exploration, higher overall maximum reward, and higher diversity of solutions in both simulated high-dimensional optimization \cite{Lazer2007} and human experiments \cite{Barkoczi2016}.

    One recent study \cite{macua2017diff} investigated the effect of communication network topology, but only as an aside, and on very small networks - they also observe improvements when using not fully-connected networks. Another work focuses on the absence of a central controller, but differs significantly in that agents solve different tasks at the same time \cite{zhang2018fully}. The performance of a ring network has also been compared to random one \cite{jing2004island}, however, there are significant differences in domains (ours is non-convex optimization, theirs is a quadratic convex assignment task), solving technique (we focus on black box optimization, theirs a GA with mutations), and the number of components (we make sure that all our networks are in a single connected component for fair comparison, while they vary the number of components). Additionally, since a ring of $N$ agents has $N-1$ edges and a random network has $O(pN^2)$ edges, their result that a sparser (ring) network performs better than a (denser) random network generally agrees with ours in terms of sparsity. Finally, the dynamic concept of adjusting connections and limiting the creation of hubs has been explored \cite{araujo2008memetic}. We have expanded on their study both by looking to the DRL context and by looking at several network topologies they left for future work. 

     To the best of our knowledge, no prior work has focused on investigating how the topology of communication between agents affects learning performance in distributed DRL, for large networks and on popular graph families.

\section{Experimental Procedure}

    \begin{figure*}[t]
         \begin{minipage}{0.32\textwidth}
                   \centering
                   \includegraphics[width=\linewidth]{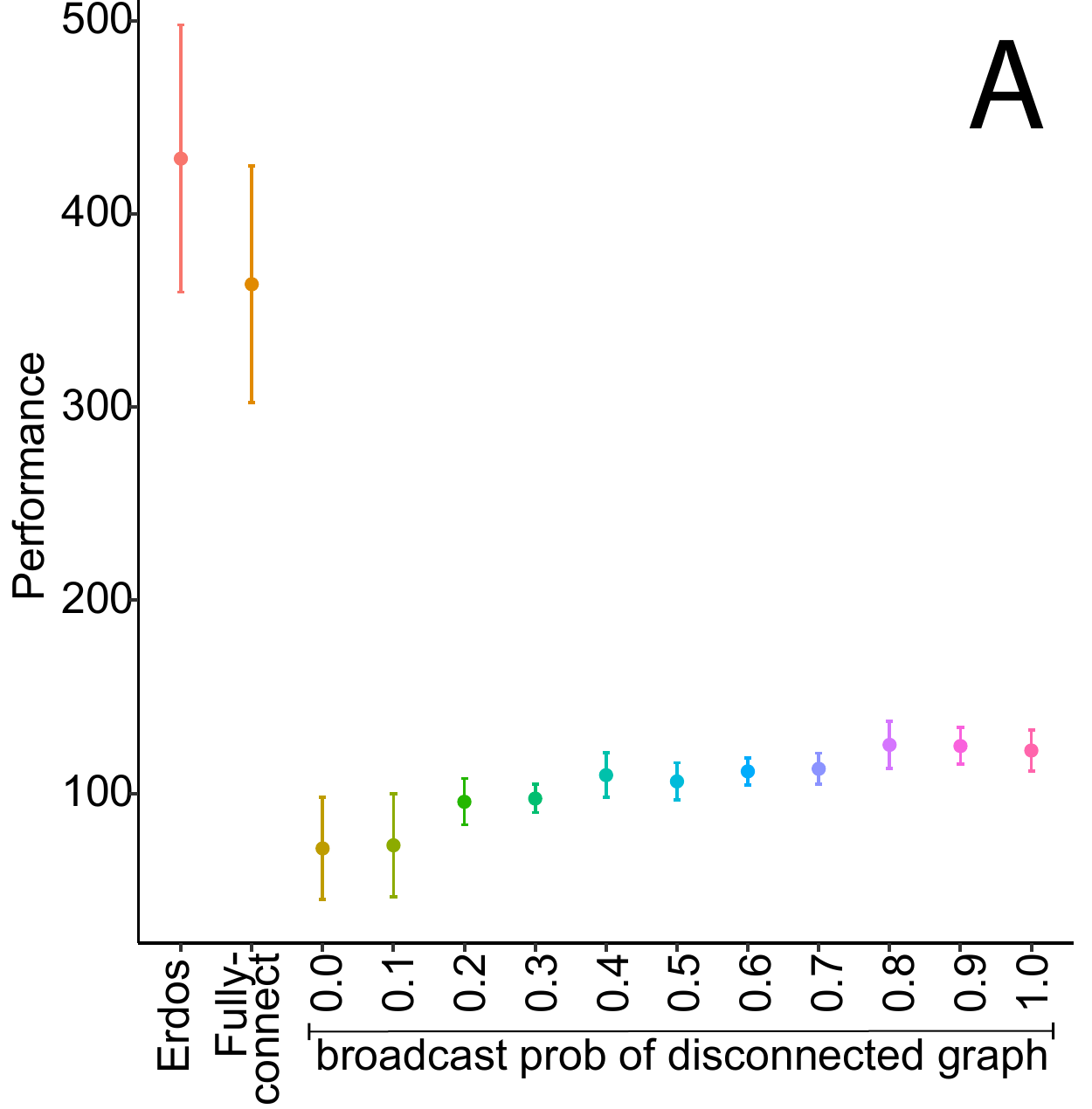}
                   \end{minipage}
                     \begin{minipage}{0.32\textwidth}
                       \centering
                       \includegraphics[width=\linewidth]{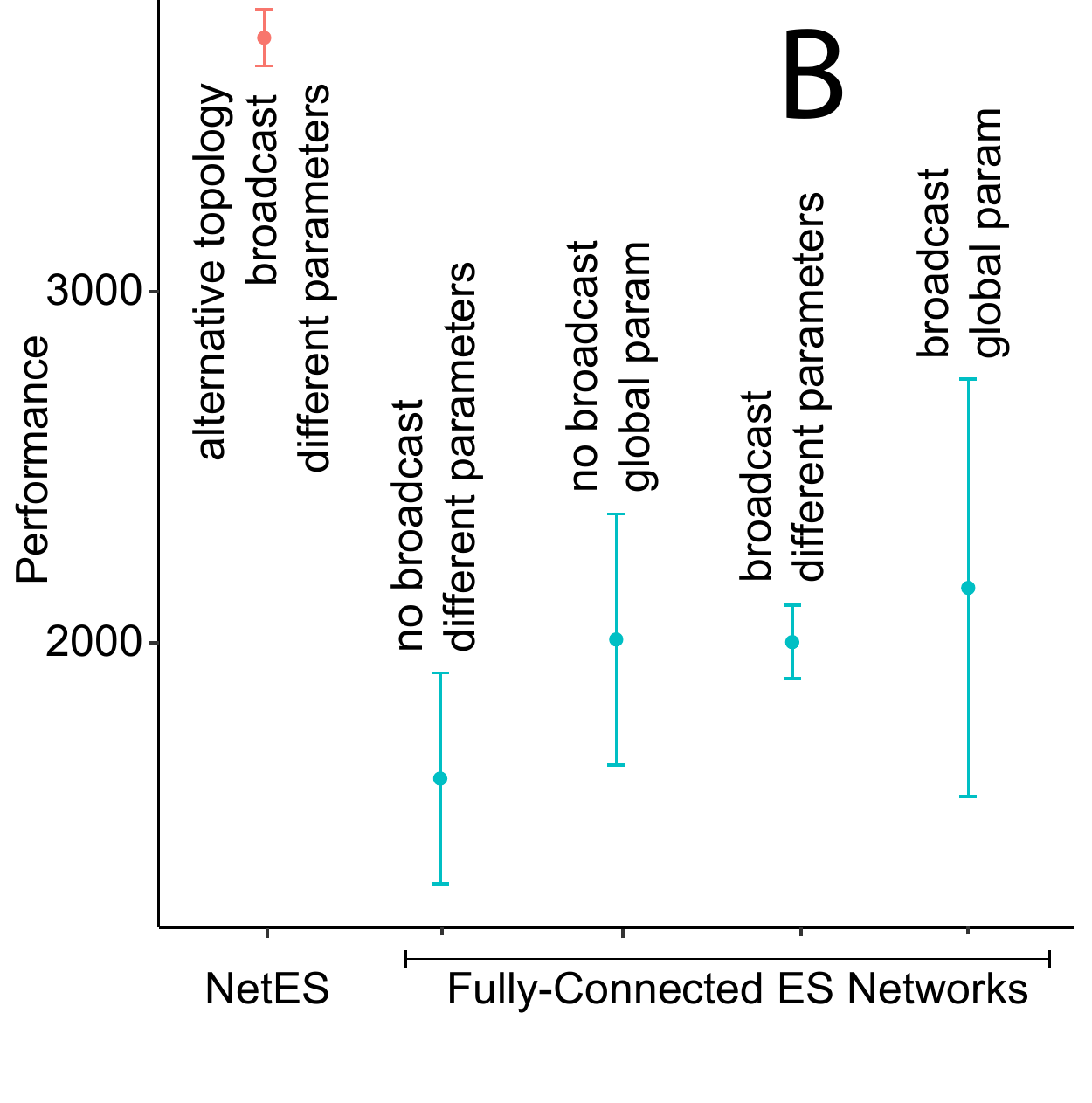}
                     \end{minipage}
                     \begin{minipage}{0.32\textwidth}
                       \centering
                       \includegraphics[width=\linewidth]{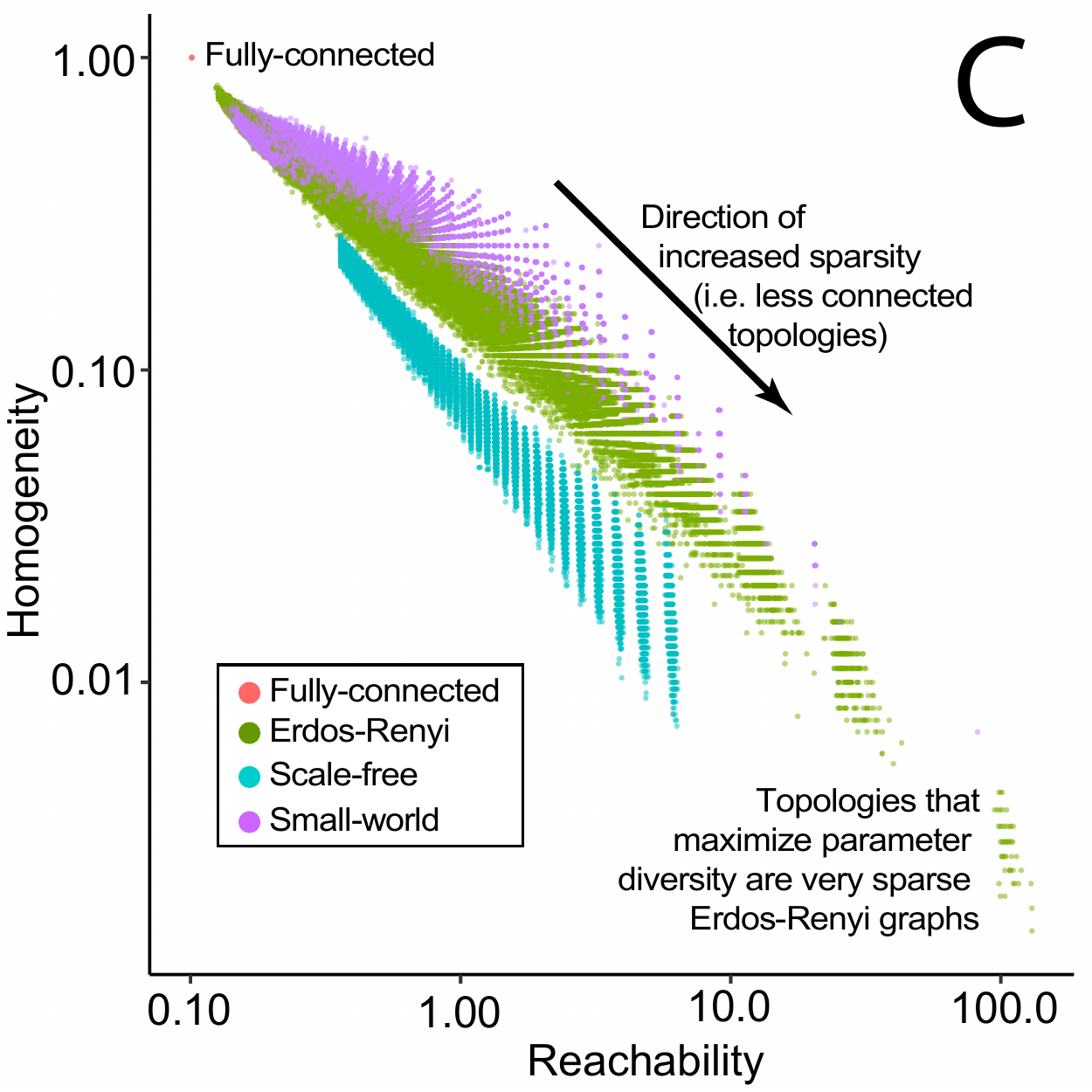}
                       \end{minipage}
        \caption{
        \textbf{A}: Agents with any amount of periodic broadcasting do not learn (RoboSchool Humanoid-v1 with 1000 agents).
        \textbf{B}: None of the control baselines with fully-connected networks learn, showing that the use of alternate topologies is what leads to learning (MuJoCo Ant-v1 with 100 agents).
        \textbf{C}: We generate instances of random networks from our four families of networks, and observe that sparser Erdos-Renyi graphs maximize the diversity of parameter updates.
        }
        \label{fig:combined2}

   \end{figure*}
   
\subsection{Goal of experiments}

The main goal of our experiments is to test our hypothesis that using alternate topologies will lead to an improvement in learning performance. Therefore, we want to be able to generate communication topologies from each of the four popular random graph families, wire our agents using this topology and deploy them to solve the DRL task at hand. We also want to run a careful ablation study to understand where the improvements come from. 

\subsection{Procedure}
    We evaluate our NetES algorithm on a series of popular benchmark tasks for deep reinforcement
    learning, selected from two frameworks---the open source Roboschool~\cite{roboschool}
    benchmark, and the MuJoCo framework~\cite{todorov2012mujoco}. The five benchmark tasks we evaluate on are:
    Humanoid-v1 (Roboschool and Mujoco), HalfCheetah-v1 (MuJoCo), Hopper-v1 (MuJoCo) and Ant-v1 (MuJoCo).
    Our choice of benchmark tasks is motivated by the difficulty of these walker-based problems.

    To maximize reproducibility of our empirical results, we use the standard evaluation metric of collecting the total
    reward agents obtain during a test-only episode, which we compute periodically during training \cite{mnih2013playing,bellemare2013arcade,Salimans2017}. Specifically, with a probability of 0.08, we intermittently pause training, take the parameters of the best agent and run this parameter (without added noise perturbation) for 1000 episodes, and take the
     average total reward over all episodes---as in Salimans et al. \cite{Salimans2017}.
     When performance eventually stabilizes to a maximum `flat' line (determined by calculating whether a 50-episode moving average has not changed by more than 5\%), we record the maximum of the evaluation performance values for this particular experimental run. As is usual \cite{bellemare2013arcade}, training performance (shown in Fig. \ref{fig:combined1}C) will be slightly lower that the corresponding maximum evaluation performance (shown in Table~\ref{tab:reward_asym}). We observe this standard procedure to be quite robust to noise.

     We repeat this evaluation procedure for multiple random instances of the same network topology by varying the random seed of network generation.
     These different instances share the same average density $p$ (i.e. the same average number of links) and the same number of nodes $N$. We use a global broadcast probability of 0.8 (a popular hyperparameter value for broadcast in optimization problems). Since each node runs the same number of episode time steps per iteration, different networks with the same $p$ can be fairly compared. For all experiments (all network families and sizes of networks), we use an average network density of 0.5 because it is sparse enough to provide good learning performance, and consistent (not noisy) empirical results.

     We then report the average performance over 6 runs with 95\% confidence intervals. We share the JSON files that fully describe our experiments and our code\footnote{JSON expement files and code implementation can be found at \href{https://github.com/d-val/NetES}{github.com/d-val/NetES}.}.

    In addition to using the evaluation procedure of Salimans et al. \cite{Salimans2017}, we also use their exact same neural network architecture:
    multilayer perceptrons with two 64-unit hidden layers separated by $tanh$ nonlinearities. We
     also keep all the modifications to the update rule introduced by Salimans et al.
      to improve performance: (1) training for one complete episode
      for each iteration; (2) employing antithetic or mirrored sampling, also known as mirrored
      sampling~\cite{geweke1988antithetic}, where we explore $\epsilon^{(t)}_i, -\epsilon^{(t)}_i$ for every sample
       $\epsilon^{(t)}_i \sim \mathcal N(0,I)$; (3) employing fitness shaping~\cite{wierstra2014natural}
        by applying a rank transformation to the returns before computing each parameter update,
        and (4) weight decay in the parameters for regularization. We also use the exact same hyperparameters as the original OpenAI (fully-connected and centralized) implementation \cite{Salimans2017}, varying only the network topology for our experiments.

\section{Results}
    \subsection{Empirical performance of network families}
    We first use one benchmark task (MuJoCo Ant-v1, because it runs fastest) and networks of 100 agents to evaluate NetES on each of the 4 families of communication topology: Erdos-Renyi, scale-free, small-world and the standard fully-connected network.
    As seen in Fig. \ref{fig:combined1}A, Erdos-Renyi strongly outperforms the other topologies. 
    
    Our hypothesis is that using alternate topologies (other than the de facto fully-connected topologies) can lead to strong improvements in learning performance. We therefore decide to focus on Erdos-Renyi graphs for all other results going forward - this choice is supported by our theoretical results which indicate that Erdos-Renyi would do better on any task. 

    If Erdos-Renyi continues to outperform fully-connected topologies on various tasks and with larger networks, our hypothesis will be confirmed - as long as they are also in agreement with our ablation studies. We leave to future work the full characterization of the performance of other network topologies. 

\subsection{Empirical performance on all benchmarks}

    Using Erdos-Renyi networks, we run larger networks of 1000 agents on all 5 benchmark results. As can be seen in Table~\ref{tab:reward_asym}, our Erdos-Renyi networks outperform fully-connected networks on all benchmark tasks, resulting in improvements ranging from 9.8\% on MuJoCo Ant-v1 to 798\% on MuJoCo Humanoid-v1. All results are statistically significant (based on 95\% confidence intervals).

    We note that the difference in performance between Erdos-Renyi and fully-connected networks is higher for smaller networks (Fig. \ref{fig:combined1}A and Fig. \ref{fig:combined2}B) compared to larger networks (Table~\ref{tab:reward_asym}) for the same benchmark, and we observe this behavior across different benchmarks. We believe that this is because NetES is able to achieve higher performance with fewer agents due to its efficiency of exploration, as supported in our empirical and theoretical results below.

\subsection{Varying network sizes}
    So far, we have compared alternate network topologies with fully-connected networks containing the same number of agents. In this section, we investigate whether organizing the communication topology using Erdos-Renyi networks can outperform larger fully-connected networks. We choose one of the benchmarks that had a small difference between the two algorithms at 1000 agents, Roboschool Humanoid-v1. As shown in Fig.~\ref{fig:combined1}B and the training curves (which display the training performance, not the evaluation metric results which would be higher as discussed earlier) in Fig.~\ref{fig:combined1}C, an Erdos-Renyi network with 1000 agents provides comparable performance to 3000 agents arranged in a fully-connected network. 

    \subsection{Ablation Study}

        To ensure that none of the modifications we implemented in the ES algorithm are causing improvements in performance, instead of just the use of alternate network topologies, we run control experiments on each modification: 1) the use of broadcast, 2) the fact that each agent/node has a different parameter set. We test all combinations.

    \subsubsection{Broadcast effect}
        We want to make sure that broadcast (over different probabilities ranging from 0.0 to 1.0) does not explain away our performance improvements. We compare `disconnected` networks, where agents can only learn from their own parameter update and from broadcasting (they do not see the rewards and parameters of any other agents each step as in NetES).  We compare them to Erdos-Renyi networks and fully-connected networks of 1000 agents on the Roboschool Humanoid-v1 task. As can be seen in Fig. \ref{fig:combined2}A practically no learning happens with \textbf{just} broadcast and no network. These experiments show that broadcast does not explain away the performance improvement we observe when using NetES.

    \subsubsection{Global versus individual parameters}
        The other change we introduce in NetES is to have each agent hold their own parameter value ${\matrixsym \theta}_i^{(t)}$ instead of a global (noised) parameter $\matrixsym \theta^{(t)}$. We therefore investigate the performance of the following 4 control baselines: fully-connected ES with 100 agent running:
        (1) same global parameter, no broadcast;
        (2) same global parameter, with broadcast;
        (3) different parameters, with broadcast;
        (4) different parameters, no broadcast; compared to NetES running an Erdos-Renyi network. For this experiment we use MuJoCo Ant-v1.  As shown in Fig. \ref{fig:combined2}B, NetES does better than all 4 other control baselines, showing that the improvements of NetES come from using alternate topologies and not from having different local parameters for each agent.

\section{Theoretical Insights}

    In this section, we present theoretical insights into why alternate topologies can outperform fully-connected topologies, and why
    Erdos-Renyi networks also outperform the other two network families we have tested. A motivating factor for introducing alternate connectivity is to search the parameter space more effectively, a common motivation in DRL and optimization in general. One possible heuristic for measuring the capacity to explore the parameter space is the diversity of parameter updates during each iteration, which can be measured by the variance of parameter updates:
    \begin{theorem}
    In a NetES update iteration $t$ for a system with $N$ agents
    with parameters $\Theta = \{\theta_1^{(t)}, ..., \theta_N^{(t)}\}$, agent communication matrix
    ${\matrixsym A} = \{a_{ij}\}$, agent-wise perturbations $\mathcal E=\{\epsilon_1^{(t)}, ..., \epsilon_N^{(t)}\}$, and parameter update  $u^{(t)}_i = \frac{\alpha}{N \sigma^2} \sum_{j=1}^N a_{ij}\cdot\big(R(\matrixsym \theta^{(t)}_j +
      \sigma \matrixsym \epsilon_j^{(t)}) \cdot ((\matrixsym \theta^{(t)}_j + \sigma \matrixsym \epsilon_j^{(t)}) - (\matrixsym \theta^{(t)}_i))\big)$ as per Equation~\ref{eqn:update_rule_4}, the following relation holds:
    \begin{multline}
    {\sf Var}_i[u^{(t)}_i] \leq \frac{\max^2 R(\cdot)}{N\sigma^4} \Big\{\Big(\frac{\lVert A^2 \rVert_F}{(\min_l |\matrixsym A_l|)^2} \Big) \cdot
    f(\Theta, \mathcal E) - \\
    \Big(\frac{\min_{l} |\matrixsym A_l|}{\max_l |\matrixsym A_l|} \Big)^2 \cdot \frac{\sigma^2}{N}(\sum_{i, j}\matrixsym \epsilon_i^{(t)}\matrixsym\epsilon_j^{(t)})\Big\}
    \end{multline}
    Here, $ |\matrixsym A_l| = \sum_j \ a_{jl}$, and $f(\Theta, \mathcal E) =$

    $\sqrt{(\sum_{j,k,m} \big((\matrixsym \theta^{(t)}_j +
    \sigma \matrixsym \epsilon_j^{(t)} - \matrixsym \theta^{(t)}_m) \cdot  (\matrixsym \theta^{(t)}_k + \sigma \matrixsym \epsilon_k^{(t)} -
    \matrixsym \theta^{(t)}_m)\big)^2)}$.
    \label{theorem:variance_bounds}
    \end{theorem} The proof for Theorem~\ref{theorem:variance_bounds} is provided in the supplementary material.

    In this work, our hypothesis is to test if some networks \textit{could} do better than the de facto fully-connected topologies used in state-of-the-art algorithms. We leave to future work the important question of optimizing the network topology for maximum performance. Doing so would require a lower bound, as it would provide us the \textit{worst-case} performance of a topology. 
    
    Instead, in this section, we are interested in providing insights into why some networks \textit{could} do better than others, which can be understood through our upper-bound, as it allows us to understand the \textit{capacity} for parameter exploration possible by a network topology.

     By Theorem~\ref{theorem:variance_bounds}, we see that the diversity of exploration in the parameter updates across
      agents is likely affected by two quantities that involve the connectivity matrix $A$:
      the first being the
      term $(\lVert A^2 \rVert_F/(\min_l |\matrixsym A_l|))^2$ (henceforth referred to as the \textit{reachability} of the network), which according to our bound we want to maximize, and the
      second being $(\min_{l} |\matrixsym A_l|/\max_l |\matrixsym A_l|)^2$ (henceforth referred to as the \textit{homogeneity} of the network), which according to our bound we
      want to be as small as possible in order to maximize the diversity of parameter updates
      across agents. Reachability and homogeneity are not independent, and are statistics of the degree distribution of a graph. It is interesting to note that the upper bound \textit{does not depend on the reward landscape $R(\cdot)$ of the task at hand}, indicating that our theoretical insights should be independent of the learning task.

    Reachability is the squared ratio of the total number of paths of length 2 in $A$ to the minimum number of links of all nodes of $A$. 
    Homogeneity is the squared ratio of the minimum to maximum connectivity of all nodes of $A$: the higher this value, the more homogeneously connected the graph is. 

    Using the above definitions for reachability and homogeneity, we generate random instances of each network family, and plot them in Fig. \ref{fig:combined2}C. Two main observations can be made from this simulation. 
    \begin{itemize}
        \item Erdos-Renyi networks maximize reachability and minimize homogeneity, which means that they likely maximize the diversity of parameter exploration. 
        \item Fully-connected networks are the single worst network in terms of exploration diversity (they minimize reachability and maximize homogeneity, the opposite of what would be required for maximizing parameter exploration according to the suggestion of our bound).
    \end{itemize}
    
    These theoretical results agree with our empirical results: Erdos-Renyi networks perform best, followed by scale-free networks, while fully-connected networks do worse.

    It is also important to note that the quantity in Theorem~\ref{theorem:variance_bounds} is not the \textbf{variance of the value function gradient}, which is typically minimized in reinforcement learning. It is instead the \textbf{variance in the positions in parameter space} of the agents after a step of our algorithm.  This quantity is more productively conceptualized as akin to a radius of exploration for a distributed search procedure rather than in its relationship to the variance of the gradient. The challenge is then to maximize the search radius of positions in parameter space to find high-performing parameters. As far as the side effects this might have, given the common wisdom that increasing the variance of the value gradient in single-agent reinforcement learning can slow convergence, it is worth noting that noise (i.e. variance) is often critical for escaping local minima in other algorithms, e.g. via stochasticity in SGD.
 
\subsection{Sparsity in Erdos-Renyi Networks}
    
    We can approximate reachability and homogeneity for Erdos-Renyi networks as a function of their density (a derivation can be found in the supplementary):
    
    \begin{lemma}
    For an Erdos-Renyi graph $\mathcal G$ with $N$ vertices, density $p$ and adjacency matrix $A$, the following approximation can be made on its Reachability $\rho(\mathcal G)$.
    \begin{equation*}
        \rho(\mathcal G) \approx (pN)^{-1/2}.
    \end{equation*}
    Similarly, its homogeneity $\gamma(\mathcal G)$ can be approximated as follows.
    \begin{equation*}
        \gamma(\mathcal G) \approx 1 - 8\sqrt{(1-p)/(Np)}.
    \end{equation*}
    \end{lemma}

    \begin{figure}[t]
              \centering
         \includegraphics[width=.7\linewidth]{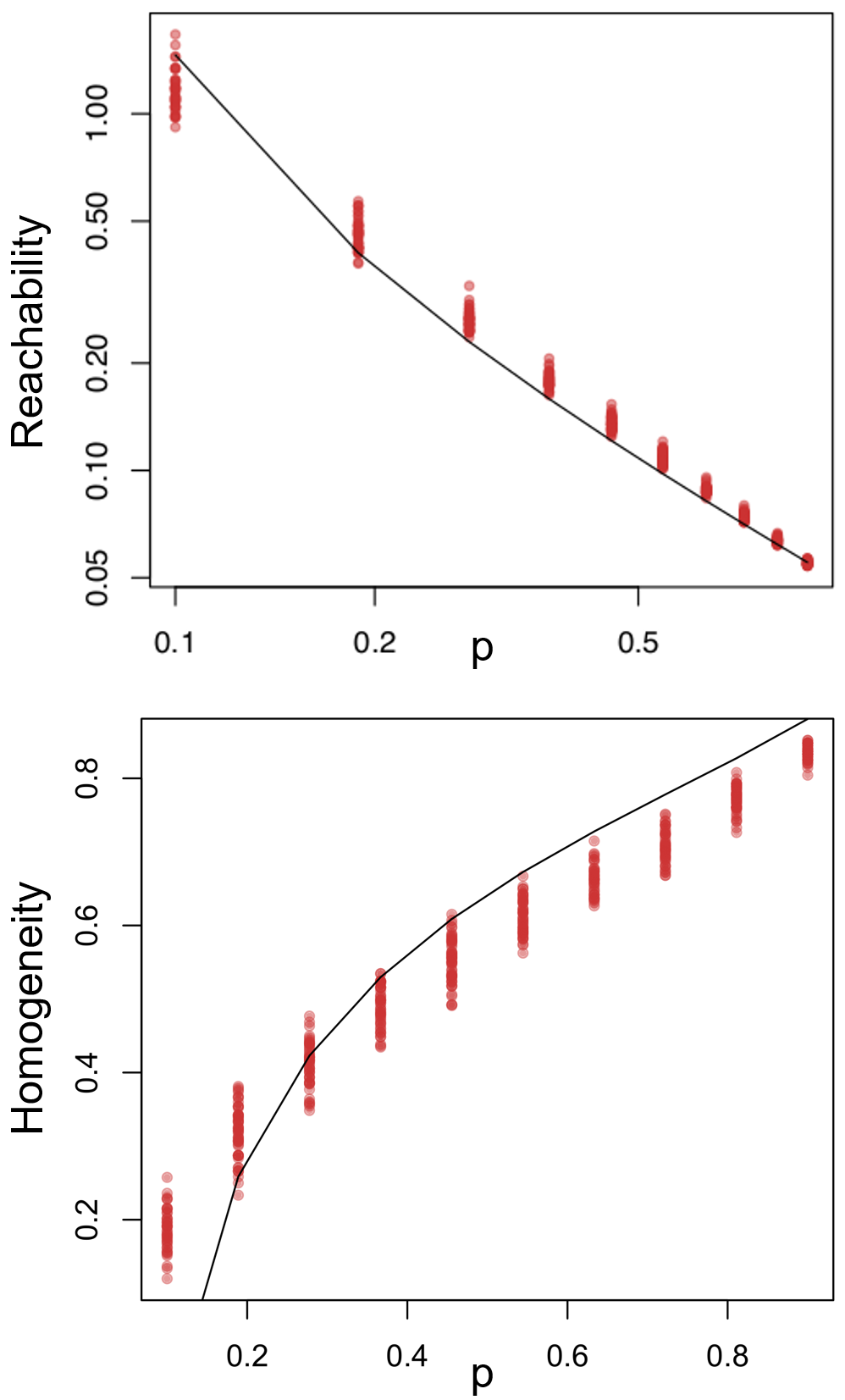}
              \caption{Reachability and homogeneity in the Erdos-Renyi case for different densities $p$. Points correspond to the real data, while the lines are the approximations. 
              }
\label{fig:erdos_homo_reach}
\end{figure}   
    
     As can be interpreted from these approximations, the sparser an Erdos-Renyi network is (i.e. the lower is $p$), the larger is the reachability and the lower is the homogeneity. The approximations and the actual reachability and homogeneity (computed directly from the graph adjacency matrix) are plotted in Fig. \ref{fig:erdos_homo_reach}. 
     
     In addition to providing insights as to why some families of network topologies do better than others (as shown in Fig. \ref{fig:combined2}C), Theorem~\ref{theorem:variance_bounds} also predicts that as Erdos-Renyi networks become sparser (less dense) -- because their reachability increases and their homogeneity decreases -- the diversity of parameter updates during each iteration would increase leading to more effective parameter search, and therefore increased performance. 
    
    By running a last number of experiments where we vary the density of Erdos-Renyi networks (keeping the number of agents at 1000) and use these topologies on the RoboSchool Humanoid-v1 DRL benchmark, we can test if sparser networks actually perform better. As can be seen in Figure \ref{joy_erdos}, when the density of Erdos-Renyi networks decreases, learning performance increases significantly. 
    
\begin{figure}[t]
\centering
  \includegraphics[width=.8\linewidth]{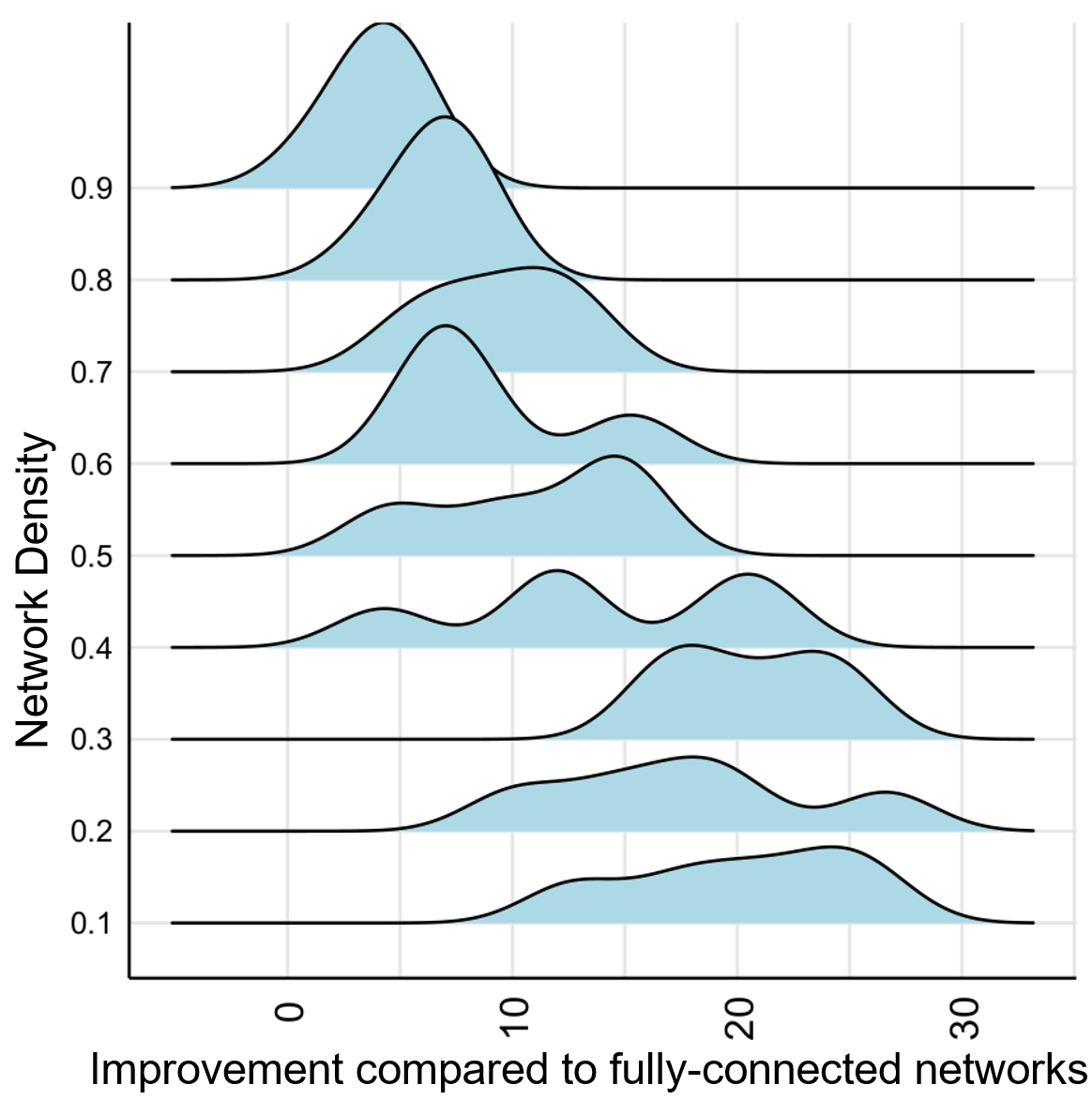}
  \caption{The distributions of reward improvements (compared to the fully-connected topologies) with network density in Erdos-Renyi networks for RoboSchool Humanoid-v1. As predicted by Theorem~\ref{theorem:variance_bounds}, as density decreases, performance increases. Note that a density of 1.0 would result in a fully-connected network. }
  \label{joy_erdos}
\end{figure}%

\section{Conclusion}
    In our work, we extended ES, a DRL algorithm, to use alternate network topologies and empirically showed that the de facto fully-connected topology performs worse in our experiments. We also performed an ablation study by running controls on all the modifications we made to the ES algorithm, and we showed that the improvements we observed are not explained away by modifications other than the use of alternate topologies. Finally, we provided theoretical insights into why alternate topologies may be superior, and observed that our theoretical predictions are in line with our empirical results. Future work could explore the use of dynamical topologies where agent connections are continuously rewired to adapt to the local terrain of the research landscape. 
    
    \section{Acknowledgements}
    The authors wish to thank Yan Leng for her help in early analysis of the properties of networks, Alia Braley for proofreading, and Tim Salimans for his help with replicating the OpenAI results as a benchmark. 
    
\bibliographystyle{ACM-Reference-Format}  
\bibliography{main}

\clearpage
\newpage

\section*{Appendix 1 : Diversity of Parameter Updates}
Here we provide proofs Theorem 1 from the main paper concerning the diversity of the parameter updates.

\begin{theorem}
In a multi-agent evolution strategies update iteration $t$ for a system with $N$ agents
with parameters $\Theta = \{\theta_1^{(t)}, ..., \theta_N^{(t)}\}$, agent communication matrix
${\matrixsym A} = \{a_{ij}\}$, agent-wise perturbations $\mathcal E=\{\epsilon_1^{(t)}, ..., \epsilon_N^{(t)}\}$, and parameter update $u^{(t)}_i$ given by the sparsely-connected update rule:
\begin{equation*}
	u^{(t)}_i = \frac{\alpha}{N \sigma^2} \sum_{j=1}^N a_{ij}\cdot\big(R(\matrixsym \theta^{(t)}_j +
	 \sigma \matrixsym \epsilon_j^{(t)}) \cdot ((\matrixsym \theta^{(t)}_j + \sigma \matrixsym \epsilon_j^{(t)}) - (\matrixsym \theta^{(t)}_i))\big) \label{eqn:update_rule}
\end{equation*}
The following relation holds:
\begin{multline}
{\sf Var}_i[u^{(t)}_i] \leq \frac{\max^2 R(\cdot)}{N\sigma^4} \Big\{\Big(\frac{\lVert A^2 \rVert_F}{(\min_l |\matrixsym A_l|)^2} \Big) \cdot
f(\Theta, \mathcal E) \\ - \Big(\frac{\min_{l} |\matrixsym A_l|}{\max_l |\matrixsym A_l|} \Big)^2 \cdot g(\mathcal E)\Big\}
\end{multline}
Here, $ |\matrixsym A_l| = \sum_j \ a_{jl}$, $f(\Theta, \mathcal E) = \Big(\sum_{j,k,m}^{N, N, N} \big((\matrixsym \theta^{(t)}_j +
\sigma \matrixsym \epsilon_j^{(t)} - \matrixsym \theta^{(t)}_m) \cdot  (\matrixsym \theta^{(t)}_k + \sigma \matrixsym \epsilon_k^{(t)} -
\matrixsym \theta^{(t)}_m)\big)^2\Big)^{\frac{1}{2}}$, and $g(\mathcal E) =
\frac{\sigma^2}{N}\Big(\sum_{i, j}^{N, N}\matrixsym \epsilon_i^{(t)}\matrixsym\epsilon_j^{(t)}\Big)$.
\label{theorem:variance_bounds}
\end{theorem}
\begin{proof}
From Equation~\ref{eqn:update_rule}, the update rule is given by:
\begin{equation}
	u^{(t)}_i = \frac{\alpha}{N \sigma^2} \sum_{j=1}^N a_{ij}\cdot\big(R(\matrixsym \theta^{(t)}_j +
	 \sigma \matrixsym \epsilon_j^{(t)}) \cdot ((\matrixsym \theta^{(t)}_j + \sigma \matrixsym \epsilon_j^{(t)}) - (\matrixsym \theta^{(t)}_i))\big)
\end{equation}
The variance of $u^{(t)}_i$ can be written as:
\begin{align}
{\sf Var}_i[u^{(t)}_i]  = \mathbb E_{i \in \mathcal A}[(u^{(t)}_i)^2] - (\mathbb E_{i \in \mathcal A}[(u^{(t)}_i)])^2
\end{align}
Expanding $\mathbb E_{i \in \mathcal A}[(u^{(t)}_i)^2]$:
\begin{equation}
= \frac{1}{N}\sum_{i \in \mathcal A}\big\{\frac{\gamma}{N\sigma^2} \sum_{j=1} \matrixsym a_{ij} \cdot R(\matrixsym \theta^{(t)}_j + \sigma \matrixsym \epsilon_j^{(t)}) \cdot (\matrixsym \theta^{(t)}_j + \sigma \matrixsym \epsilon_j^{(t)} - \matrixsym \theta^{(t)}_i)\big\}^2
\end{equation}
Simplifying:
\begin{multline}
= \frac{1}{N\sigma^4}\sum_{i,j,k} \Big( \frac{\matrixsym a_{ij} \matrixsym a_{ik}}{|\matrixsym A_i|^2} R(\matrixsym \theta^{(t)}_j + \sigma \matrixsym \epsilon_j^{(t)}) R(\matrixsym \theta^{(t)}_k + \sigma \matrixsym \epsilon_k^{(t)}) \\ \cdot (\matrixsym \theta^{(t)}_j + \sigma \matrixsym \epsilon_j^{(t)} - \matrixsym \theta^{(t)}_i) \cdot (\matrixsym \theta^{(t)}_k + \sigma \matrixsym \epsilon_k^{(t)} - \matrixsym \theta^{(t)}_i) \Big)
\end{multline}
Since $R(\cdot) \leq \max R(\cdot)$, therefore:
\begin{equation}
\leq \frac{\max^2 R(\cdot)}{N\sigma^4}\sum_{i,j,k} \frac{\matrixsym a_{ij} \matrixsym a_{ik}}{|\matrixsym A_i|^2} \cdot (\matrixsym \theta^{(t)}_j + \sigma \matrixsym \epsilon_j^{(t)} - \matrixsym \theta^{(t)}_i) \cdot  (\matrixsym \theta^{(t)}_k + \sigma \matrixsym \epsilon_k^{(t)} - \matrixsym \theta^{(t)}_i)
\end{equation}
\begin{equation}
\leq \frac{\max^2 R(\cdot)}{N\sigma^4}\sum_{i,j,k} \frac{\matrixsym a_{ij} \matrixsym a_{ik}}{\min_l |\matrixsym A_l|^2} \cdot (\matrixsym \theta^{(t)}_j + \sigma \matrixsym \epsilon_j^{(t)} - \matrixsym \theta^{(t)}_i) \cdot  (\matrixsym \theta^{(t)}_k + \sigma \matrixsym \epsilon_k^{(t)} - \matrixsym \theta^{(t)}_i)
\end{equation}
By the Cauchy-Schwarz Inequality:
\begin{multline}
\mathbb E_{i \in \mathcal A}[(u^{(t)}_i)^2]  \leq \frac{\max^2 R(\cdot)}{N\sigma^4} \Big(\sum_{i,j,k} \frac{(\matrixsym a_{ij} \matrixsym a_{ik})^2}{\min_l |\matrixsym A_l|^4} \Big)^{\frac{1}{2}} \\ \cdot \Big(\sum_{i,j,k} \big((\matrixsym \theta^{(t)}_j + \sigma \matrixsym \epsilon_j^{(t)} - \matrixsym \theta^{(t)}_i) \cdot  (\matrixsym \theta^{(t)}_k + \sigma \matrixsym \epsilon_k^{(t)} - \matrixsym \theta^{(t)}_i)\big)^2\Big)^{\frac{1}{2}}
\end{multline}
Since $\matrixsym a_{ij} \in \{0,1\} \forall \ (i,j), (\matrixsym a_{ij} \matrixsym a_{ik})^2 = \matrixsym a_{ij} \matrixsym a_{ik}\ \forall (i,j,k)$. Additionally, we know that $\matrixsym a_{ij} = \matrixsym a_{ji}$, since $\matrixsym A$ is symmetric. Therefore, $\sum_i \matrixsym a_{ij} \matrixsym a_{ik} = \sum_i \matrixsym a_{ji} \matrixsym a_{ik} = \matrixsym A^2_{jk}$. Using this:
\begin{multline}
\mathbb E_{i \in \mathcal A}[(u^{(t)}_i)^2]  \leq \frac{\max^2 R(\cdot)}{N\sigma^4} \cdot \Big(\frac{|\matrixsym A^2|^\frac{1}{2}}{\min_l |\matrixsym A_l|^2} \Big) \\ \cdot \Big(\sum_{i,j,k} \big((\matrixsym \theta^{(t)}_j + \sigma \matrixsym \epsilon_j^{(t)} - \matrixsym \theta^{(t)}_i) \cdot  (\matrixsym \theta^{(t)}_k + \sigma \matrixsym \epsilon_k^{(t)} - \matrixsym \theta^{(t)}_i)\big)^2\Big)^{\frac{1}{2}}
\end{multline}
Replacing $\Big(\sum_{i,j,k} \big((\matrixsym \theta^{(t)}_j + \sigma \matrixsym \epsilon_j^{(t)} - \matrixsym \theta^{(t)}_i) \cdot  (\matrixsym \theta^{(t)}_k + \sigma \matrixsym \epsilon_k^{(t)} - \matrixsym \theta^{(t)}_i)\big)^2\Big)^{\frac{1}{2}} = f(\Theta, \mathcal E)$, where $\Theta = \{\matrixsym \theta^{(t)}_i\}_{i=1}^N, \mathcal E = \{\matrixsym \epsilon_i\}_{i=1}^N$ for compactness, we obtain:
\begin{equation}
\small
\mathbb E_{i \in \mathcal A}[(u^{(t)}_i)^2]  \leq \frac{\max^2 R (\cdot)}{N\sigma^4} \cdot \Big(\frac{|\matrixsym A^2|^\frac{1}{2}}{\min_l |\matrixsym A_l|^2} \Big) \cdot  f(\Theta, \mathcal E)
\end{equation}
Similarly, the squared expectation of $(u^{(t)}_i)$ over all agents can be given by:
\begin{multline}
(\mathbb E_{i \in \mathcal A}[u^{(t)}_i])^2 = \Big(\frac{1}{N}\sum_{i \in \mathcal A}\big\{\frac{\gamma}{N\sigma^2} \sum_{j=1} \matrixsym a_{ij} \cdot R(\matrixsym \theta^{(t)}_j + \sigma \matrixsym \epsilon_j^{(t)})\\ \cdot (\matrixsym \theta^{(t)}_j + \sigma \matrixsym \epsilon_j^{(t)} - \matrixsym \theta^{(t)}_i)\big\}\Big)^2
\end{multline}
\begin{multline}
= \frac{1}{N^2\sigma^4}\Big(\sum_{i \in \mathcal A}\big\{\frac{1}{|\matrixsym A_i|} \sum_{j=1} \matrixsym a_{ij} \cdot R(\matrixsym \theta^{(t)}_j + \sigma \matrixsym \epsilon_j^{(t)})  \cdot (\matrixsym \theta^{(t)}_j + \sigma \matrixsym \epsilon_j^{(t)} - \matrixsym \theta^{(t)}_i)\big\}\Big)^2
\end{multline}
\begin{multline}
= \frac{1}{N^2\sigma^4}\Big(\sum_{i, j}\big\{\frac{\matrixsym a_{ij}}{|\matrixsym A_i|} \cdot R(\matrixsym \theta^{(t)}_j + \sigma \matrixsym \epsilon_j^{(t)}) \cdot (\matrixsym \theta^{(t)}_j + \sigma \matrixsym \epsilon_j^{(t)} - \matrixsym \theta^{(t)}_i)\big\}\Big)^2 \\
\end{multline}
Since $R(\cdot) \geq \min R(\cdot)$, therefore:
\begin{align}
\geq \frac{\min^2 R (\cdot)}{N^2\sigma^4} \Big(\sum_{i, j}\big\{\frac{\matrixsym a_{ij}}{|\matrixsym A_i|} \cdot (\matrixsym \theta^{(t)}_j + \sigma \matrixsym \epsilon_j^{(t)} - \matrixsym \theta^{(t)}_i)\big\}\Big)^2
\end{align}
\begin{align}
\geq \frac{\min^2 R (\cdot)}{N^2\sigma^4 \max_l |\matrixsym A_l|^2} \Big(\sum_{i, j}\big\{\matrixsym a_{ij}\cdot (\matrixsym \theta^{(t)}_j + \sigma \matrixsym \epsilon_j^{(t)} - \matrixsym \theta^{(t)}_i)\big\}\Big)^2
\end{align}
Since $\matrixsym A$ is symmetric, $\sum_{i,j}^{N,N} \matrixsym a_{ij} \cdot (\matrixsym \theta^{(t)}_j + \sigma \matrixsym \epsilon_j - \matrixsym \theta^{(t)}_i) = \sum_{i,j}^{N,N} \matrixsym a_{ij} \cdot (\matrixsym \theta^{(t)}_i + \sigma \matrixsym \epsilon_i - \matrixsym \theta^{(t)}_j)$. Therefore:
\begin{multline}
=\frac{\min^2 R (\cdot)}{N^2\sigma^4 \max_l |\matrixsym A_l|^2} \Big(\sum_{i, j}\frac{1}{2}\big\{\matrixsym a_{ij}\cdot (\matrixsym \theta^{(t)}_j + \sigma \matrixsym \epsilon_j^{(t)} - \matrixsym \theta^{(t)}_i) \\ + \matrixsym a_{ij}\cdot (\matrixsym \theta^{(t)}_i + \sigma \matrixsym \epsilon_i^{(t)} - \matrixsym \theta^{(t)}_j)\big\}\Big)^2
\end{multline}
Therefore,
\begin{multline}
(\mathbb E_{i \in \mathcal A}[u^{(t)}_i])^2 = \\
\frac{\min^2 R (\cdot)}{N^2\sigma^2 \max_l |\matrixsym A_l|^2} \Big(\sum_{i, j}\frac{1}{2}\big\{\matrixsym a_{ij}\cdot
(\matrixsym \epsilon_j^{(t)} + \matrixsym \epsilon_i^{(t)})\big\}\Big)^2
\end{multline}
Using the symmetry of $\matrixsym A$, we have that $\sum_{i,j}^{N,N} \matrixsym a_{ij} \matrixsym \epsilon_i = \sum_{i,j}^{N,N} \matrixsym a_{ij} \matrixsym \epsilon_j$. Therefore:
\begin{align}
&= \frac{\min^2 R (\cdot)}{N^2\sigma^2 \max_l |\matrixsym A_l|^2} \Big(\sum_{i, j}\matrixsym a_{ij}\cdot \matrixsym \epsilon_j^{(t)}\Big)^2 \\
& = \frac{\min^2 R (\cdot)}{N^2\sigma^2 \max_l |\matrixsym A_l|^2} \Big(\sum_{j}|\matrixsym A_{j}|\cdot \matrixsym \epsilon_j^{(t)}\Big)^2 \\
&\geq \frac{\min^2 R(\cdot) \min_{l} |\matrixsym A_l|^2}{N^2\sigma^2 \max_l |\matrixsym A_l|^2} \Big(\sum_{i, j}\matrixsym \epsilon_i^{(t)}\matrixsym\epsilon_j^{(t)}\Big)
\end{align}
Combining both terms of the variance expression, and using the normalization of the iteration rewards that ensures $\min R(\cdot) = -\max R(\cdot)$, we can obtain (using $g(\mathcal E) = \frac{\sigma^2}{N}\Big(\sum_{i, j}\matrixsym \epsilon_i^{(t)}\matrixsym\epsilon_j^{(t)}\Big)$):
\begin{multline}
{\sf Var}_{i \in \mathcal A}[u^{(t)}_i] \leq \frac{\max^2 R(\cdot)}{N\sigma^4} \Big\{\Big(\frac{|\matrixsym A^2|^\frac{1}{2}}{\min_l |\matrixsym A_l|^2} \Big) \cdot  f(\Theta, \mathcal E) \\ - \Big(\frac{\min_{l} |\matrixsym A_l|^2}{\max_l |\matrixsym A_l|^2} \Big) \cdot g(\mathcal E)\Big\}
\end{multline}
\end{proof}

\section*{Appendix 2 : Approximating Reachability and Homogeneity for Large Erdos-Renyi Graphs}
Recall that a Erdos-Renyi graph is constructed in the following way
\begin{enumerate}
\item Take $n$ nodes
\item For each pair of nodes, link them with probability $p$
\end{enumerate}
\begin{figure*}[th!]
\begin{center}
\includegraphics[width=.6\textwidth]{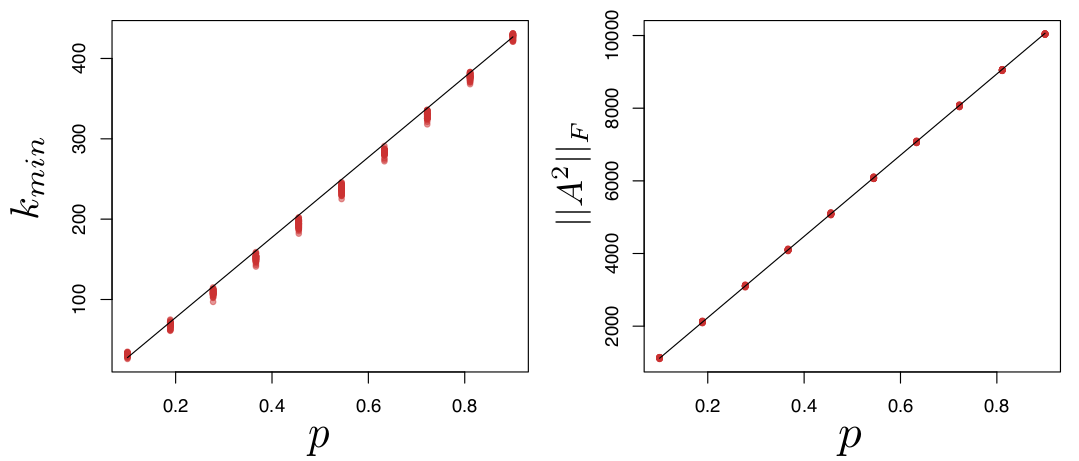}
\caption{Comparison between the values of $k_{min}, ||A^2||_F,$ and Reachability as a function of $p$ for different realizations of the Erdos-Renyi model (points) and their approximations given in Equations (\ref{approxkmin}), (\ref{approxa2}) and (\ref{approxbound1}) respectively (lines).}
\label{default}
\end{center}
\end{figure*}
The model is simple, and we can infer the following:
\begin{itemize}
\item The average degree of a node is $p(n-1)$
\item The distribution of degree for the nodes is the Binomial distribution of $n-1$ events with probability $p$, $B(n-1,p)$.
\item The (average) number of paths of length 2 from one node $i$ to a node $j\neq i$ ($n^{(2)}_{ij}$) can be calculated this way: a path of length two between $i$ and $j$ involves a third node $k$. Since there are $n-2$ of them, the maximun number of paths between $i$ and $j$ is $n-2$. However, for that path to exists there has to be a link between $i$ and $k$ and $k$ and $j$, an event with probability $p^2$. Thus, the average number of paths between $i$ and $j$ is $p^2(n-2)$
\end{itemize}
\subsection*{Estimating Reachability}
We can then estimate Reachability:
$$
Reachability = \frac{||A^2||_F}{(min_l |A_l|)^2} = \frac{\sqrt{\sum_{i,j} n^{(2)}_{ij}}}{k_{min}^2}
$$
where $k_{min} = (min_l |A_l|)$ is the minimum degree in the network. Given the above calculations we can approximate
$$
\sum_{i,j} n^{(2)}_{ij} = \sum_{i} n^{(2)}_{ii} + \sum_{i\neq j} n^{(2)}_{ij} \approx n \times [p(n-1)] + n(n-1) \times [p^2(n-2)]
$$
where the first term is the number of paths of length 2 from $i$ to $i$ summed over all nodes, i.e. the sum of the degrees in the network. The second term is the sum of $p^2(n-2)$ for the terms in which $i\neq j$. For large $n$ we have that
$$
\sum_{i,j} n^{(2)}_{ij} \approx p^2 n^3
$$
and thus,
\begin{equation}\label{approxa2}
||A^2||_F \approx \sqrt{p^2n^3}.
\end{equation}

For the denominator $k_{min}$ we could use the distribution of the minimum of the binomial distribution $B(n-1,p)$. However, since it is a complicated calculation we can approximate this way: since the binomial distribution $B(n-1,p)$ looks like a Gaussian, we can say that the minimum of the distribution is closed to the mean minus two times the standard deviation:

\begin{equation}\label{approxkmin}
k_{min} \approx p(n-1) - 2 \sqrt{p(n-1)(1-p)}
\end{equation}

Once again in the case of large $n$ we have

$$k_{min} \approx pn$$

Thus

\begin{equation}\label{approxbound1}
Reachability \approx \frac{\sqrt{p^2 n^3}}{[p(n-1) - 2 \sqrt{p(n-1)(1-p)}]^2}
\end{equation}

Assuming that $n$ is large, we can approximate

$$
Reachability \approx \frac{p n^{3/2}}{p^2n^2} = \frac{1}{pn^{1/2}}
$$

Thus the bound decreases with increasing $n$ and $p$. Note that the density of the Erdos-Renyi graph (the number of links over the number of possible links) is $p$. And thus for a fixed $n$ more sparse networks $p \simeq 0$ have larger Reachability than more connected networks $p\simeq 1$.
\subsection*{Estimating Homogeneity}

The Homogeneity is defined as

$$
Homogeneity = \left(\frac{k_{min}}{k_{max}}\right)^2
$$

As before we can approximate

$$
k_{max} \approx p(n-1) + 2 \sqrt{p(n-1)(1-p)}
$$

And thus

$$
Homogeneity \approx \left(\frac{p(n-1) - 2 \sqrt{p(n-1)(1-p)}}{p(n-1) + 2 \sqrt{p(n-1)(1-p)}} \right)^2
$$

For large $p$ we can approximate it to be

\begin{equation}\label{approxHomogeneity}
Homogeneity \approx 1 - 8 \frac{\sqrt{1-p}}{\sqrt{np}}
\end{equation}

which shows that for $p \simeq 1$ we have that Homogeneity grows as a function of $p$. Thus for fixed number of nodes $n$, increasing $p$ we get larger values of the Homogeneity.

\end{document}